\documentclass[12pt]{article}
\usepackage{amssymb}
\usepackage{amsmath,amsthm}
\usepackage{amsfonts}
\usepackage{graphicx}
\usepackage{graphics}
\usepackage{xcolor}
\usepackage{algorithm}
\usepackage{algpseudocode}
\usepackage[numbers, sort&compress]{natbib}
\numberwithin{equation}{section}
\numberwithin{footnote}{section}
\newtheorem{Theorem}{Theorem}[section]

\newtheorem{cor}{Corollary}[section]

\newcommand{\bp}{\begin{proof}}
\newcommand{\ep}{\end{proof}}
\newcommand{\be}{\begin{equation}}
\newcommand{\ee}{\end{equation}}
\newcommand{\bes}{\begin{equation*}}
\newcommand{\ees}{\end{equation*}}
\newcommand{\mbf}{\mathbf}

\begin{document}
\date{\small\textsl{\today}}
\title{
Local Binary Pattern (LBP) Optimization \\ for Feature Extraction
}
\author{ Zeinab Sedaghatjoo$^1$
\begin{footnote}{Corresponding author.\newline {\em  E-mail addresses:}
\newline z.sedaqatjoo@aut.ac.ir , zeinab.sedaghatjoo@gmail.com (Z. Sedaghatjoo).  \\
 h{\_}hosseinzadeh@aut.ac.ir  ,  hosseinzadeh@pgu.ac.ir (H. Hosseinzadeh). \\
b\_sadeghi\_b@alzahra.ac.ir (B. S. Bigham).
}$\vspace{.2cm} $
\end{footnote} 
, Hossein Hosseinzadeh$^1$, Bahram Sadeghi Bigham$^2$
\\
\small{\em 
 $^1$ Department of Mathematics, Persian Gulf University, Bushehr, Iran.} \\
\small{\em  
 $^2$ Department of Computer Science, Faculty of Mathematical Sciences, Alzahra University, Tehran, Iran.}
\vspace{-1mm}} \maketitle
\vspace{.9cm}

\begin{abstract}
The rapid growth of image data has led to the development of advanced image processing and computer vision techniques, which are crucial in various applications such as image classification, image segmentation, and pattern recognition. Texture is an important feature that has been widely used in many image processing tasks. Therefore, analyzing and understanding texture plays a pivotal role in image analysis and understanding.Local binary pattern (LBP) is a powerful operator that describes the local texture features of images. This paper provides a novel mathematical representation of the LBP by separating the operator into three matrices, two of which are always fixed and do not depend on the input data. These fixed matrices are analyzed in depth, and a new algorithm is proposed to optimize them for improved classification performance. The optimization process is based on the singular value decomposition (SVD) algorithm. As a result, the authors present optimal LBPs that effectively describe the texture of human face images.  Several experiment results presented in this paper convincingly verify the efficiency and superiority of the optimized LBPs for face detection and facial expression recognition tasks.
\vspace{.5cm}
\end{abstract}

\section{Introduction}
The Local Binary Pattern (LBP) is a texture descriptor widely used in computer vision for image classification. Initially introduced by Ojala et al. \cite{ojala2002multiresolution}, LBP has become popular for its ability to extract texture features effectively while maintaining computational simplicity. The core concept of the LBP involves comparing each pixel of an image with its neighboring pixels to encode the local texture information into binary patterns. The resulting binary values obtained from the comparisons are concatenated sequentially in a clockwise order, forming an 8-digit binary number for each pixel. A histogram is computed over the entire pixels, capturing the frequency of different LBP patterns. The histogram generated from LBP serves as a feature vector for the image and can be directly utilized for classification purposes.

Variations of LBP have been developed to enhance its performance and address specific challenges in different applications \cite{zhou2018pose, kaplan2020improved, tekin2020new, vu2022masked}. A comparative study on the LBP based on face recognition is prsented in \cite{yang2013comparative}. The enhanced local binary pattern histogram (ELBPH) is analysed there and it is highlighted that it enhances the LBP results efficiently. In ELBPH an image is divided into some regions (sub-images), then a regional LBP histogram is extracted from each region and finally concatenates all the regional histograms into a single global histogram as a feature. Figure \ref{fig32} shows this process, graphically. In \cite{karis2016local} varients of LBP method and modifications are studied and analysed in object detection. As a result, the LBP method is too sensitive in details of the image retrieve by the system. In term of the object detection, the same objects can result different LBP values and the system might be confuse to determine the object. So, the LBP needs a refinement for object detection. 
Recently, a review is provided in \cite{khaleefah2020review} on the LBP and its modifications. The paper focuses the current trends for using, modifying and adapting the LBP in the image processing for feature extraction. And a robust LBP is proposed in \cite{karanwal2022robust} for face recognition in different challenges.  Two new descriptors are launched there to overcome noisy thresholding function in the LBP. So, the center pixel of a 3×3 patch is replaced by the mean of the patch. A scale and pattern adaptive local binary pattern (SPALBP) is proposed in \cite{hu2024scale} to overcome rotation changes or noise corruptions. Also, several papers are devoted to the use of the LBP for object detection \cite{baskar2023vision, lan2024neighbourhood, luo2023texture, karanwal2023triangle}.


In the literature, it has been shown that  8-digit binary numbers where the zeros and ones are in two disjoint regions are more useful for certain applications \cite{ojala2002multiresolution, khaleefah2020review}. Then the other binary numbers contain some noise. These useful patterns are named uniform LBPs \cite{khaleefah2020review}. A binary pattern contains at most two transitions from 0 to 1 or 1 to 0 in if it is uniform. For example, 00011110 (has two transitions) is a uniform pattern, but 01011100  (has 4 transitions) is not. Mathematically, the uniform patterns primarily capture directional derivatives, which can highlight the direction of the light in the image. Therefore, using the uniform LBP binary numbers appears to be beneficial.
In this study, we explain the LBP feature extraction process from a mathematical perspective and propose a new algorithm to find optimal LBP features for classification tasks. This algorithm leverages the analysis of the singular value decomposition (SVD) of the LBP matrix, introduced in here.
The experimental results presented in this paper clearly demonstrate the effectiveness of the optimal LBP features in improving classification performance compared to the standard LBP features. 
The rest of the paper is organized as follows:
\begin{itemize}
\item
In Section \ref{Sec2}, the LBP algorithm is introduced and the process of creating LBP histogram is analyzed. Then, the main mathematical properties of the histogram representation is described there.
\item
In Section \ref{Sec3}, the LBP is explained in matrix form, and it is shown that the LBP features of an image can be extracted by evaluating three matrices; LBP  Matrix, Tiling matrix, and Histogram matrix. These matrices are denoted by $E, L$ and $H$, respectively.
\item
Section \ref{Sec4} is dedicated to highlighting the motivation and novelty of the paper. Then, the role of matrices $L$ and $H$ in the feature extraction is described there.
\item
In Section \ref{Sec5}, some mathematical studies are presented to find optimal values of matrices $L$ and $H$ for the problems. This optimization is done by the use of SVD algorithm.
\item
In Section \ref{Sec6}, a novel algorithm is proposed to identify optimal LBP values for the feature extraction.
\item
In Section \ref{Sec7}, results of several numerical experiments are presented, which highlight the efficiency of the proposed algorithm in tasks such as face detection and facial expression recognition.
\item
Finally, the paper concludes with a concise conclusion in Section \ref{Sec8}.
\end{itemize}

\section{LBP}\label{Sec2}
Local Binary Pattern (LBP) is a texture descriptor commonly used in computer vision. It operates on images by assigning each pixel a binary code based on comparisons with its neighbouring pixels \cite{ojala2002multiresolution}. To apply LBP, an image is divided into several local regions, and the LBP features are extracted from each region, sequentially. These LBP features are then concatenated to form a global description of the image. 
\begin{figure}[h!]
\centerline{
\includegraphics[scale=0.55]{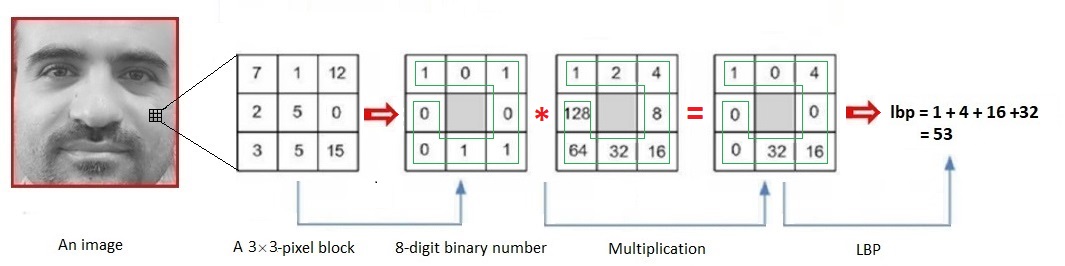}\hspace{.5cm}
}
\caption{LBP values computation process. Each $3\times 3$ pixel block in the image is encoded to a LBP value.}
\label{im11}
\end{figure}
As shown in Figure \ref{im11},  LBP typically operates on $3 \times 3$-pixel blocks, where the difference between the central pixel and its eight adjacent pixels is extracted as the local texture feature representation.  
This difference is captured by the function $S$, defined as follows:
\begin{equation}\label{s}
S(p,c)=
\left \{
\begin{array}{ll}
1, ~~~~~~~~~~~ \hbox{if} ~~ g(p) \geq g(c),\\
\\
0, ~~~~~~~~~~~ \hbox{if} ~~ g(p) < g(c).
\end{array}\right.
\end{equation}
where $g(p)$ and $g(c)$ are the values of the neighbouring pixel $p$ and the central pixel $c$, respectively. Then, if we label the neighbouring pixels as $p_1, p_2, ..., p_8$
, then an 8-bit local binary pattern centered at $c$ will be encoded as follows:
\begin{equation}\label{lbp}
lbp(c)=\sum_{i=1}^{8} 2^{i-1} S(p_i,c).
\end{equation}
After obtaining the $lbp$ values for all pixels in each local region of the image, histograms are generated to capture the occurrence frequency  of  different LBP features within the region. This histogram effectively summarizes the local texture patterns present in the region. Generally, there are $2^8=256$ patterns in the histogram, but it is often summarized into $2^3=8$ to $2^4=16$ patterns for simplicity \cite{ojala2002multiresolution}. These histograms are combined into a unified histogram, as shown in Figure \ref{fig32}, representing the overall texture characteristics of the image. 
\begin{figure}[h!]
\centerline{
\includegraphics[scale=0.6]{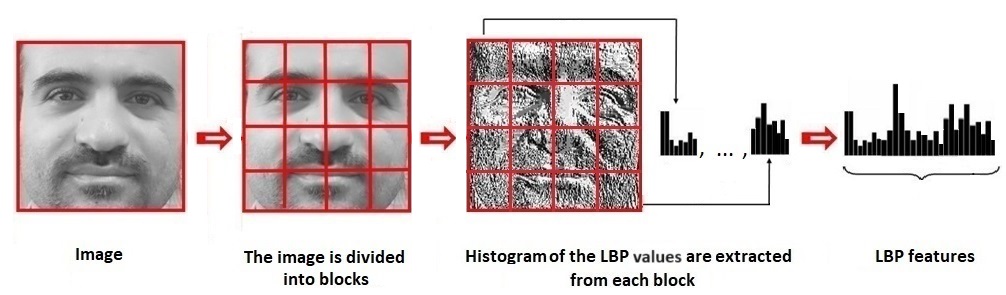}\hspace{.5cm}
}
\caption{LBP feature extraction process. The image is divided into 16 sub-regions and LBP values are extracted from each region. Then the histogram is applied on the LBP values to extract $8$ features for each sub-region. Then, finally, $128=16 \times 8$ features are exploited from the image as LBP features.}
\label{fig32}
\end{figure}

\subsection{Inverse of $lbp$}
The function presented in Equation \eqref{lbp}, $lbp$, transforms an 8-bit binary number to interval $[0,255]$. This transformation is both one to one and onto, meaning each 8-bit binary number is uniquely linked to only a single integer number within the interval, and vice versa. Consequently, there exists an inverse function, $lbp^{-1}$, which maps integers in $[0,255]$ back to the 8-bit binary numbers. Algorithm \ref{alg1} represents the inverse function. Let $a=[S(p_1,c), S(p_2,c), ..., S(p_8,c)]$ represents an 8-bit binary number. It is clear that $lbp^{-1}$ maps $0$ to $[0, 0, ...,0]$ and $255$ to $[1, 1, ..., 1]$.  And the other numbers between $0$ and $255$ are mapped to their corresponding 8-bit binary numbers as shown in  Algorithm \ref{alg1}. 

\begin{algorithm}
\caption{: Function $lbp^{-1}$.}\label{alg1}
\begin{algorithmic}[1]
\State \textbf{Input:} $lbp \in [0, 255]$.
\State Initialize vector $a$ as $a = [0, 0, 0, 0, 0, 0, 0, 0]$.
\For{$i = 7$ to $0$}
    \If{$lbp \geq 2^{i}$}
        \State Set $a[i] = 1$.
        \State Update $lbp$ value as $lbp \leftarrow lbp - 2^{i}$.
    \EndIf
\EndFor
\State \textbf{Output:} return vector $a$.
\end{algorithmic}
\end{algorithm}

From Algorithm \ref{alg1}, the eighth bit of the 8-bit binary number is set to $0$ if $lbp < 2^7$ and to  $1$ if $lbp \geq 2^7$. 
When the eighth bit is $0$, the seventh bit is set to $0$ if $lbp<2^6$ and to $1$ if $lbp \geq 2^6$. And when the eighth bit is $1$, the seventh bit is set to $0$ if $lbp-2^7<2^6$ and to $1$ if $lbp-2^7 \geq 2^6$. 
This rule can be extended to the other bits of the 8-bit binary number and find a simple and intuitive transformation from $lbp$ values to 8-bit binary numbers. Figure \ref{fig1} illustrates the inverse transformation for evaluating the four last bits of the 8-bit binary numbers, graphically. This process can be extended to evaluate the four first bits of 8-bit binary numbers as well. This evaluation can notably affect the effectiveness of the LBP feature extraction, particularly when the histogram is used to reduce the number of the features. The next subsection delves into this subject in detail.

\subsection{Histogram and its restriction }\label{Sec31} 
Considering that extracting $256$ features from LBP is not desirable, they are typically summarized to $16$ or fewer features using histograms \cite{ojala2002multiresolution, martolia2020modified} for dimension reduction.  
The histogram clusters LBP features by dividing the range of the $lbp$ values into equal sub-intervals, thereby reducing the number of features.  In Figure \ref{fig1}, the $lbp$ values are divided into  $m=2, 4, 8$ and $16$ parts from top to bottom. In this figure, only the eighth cell of vector $a$ is meaningful when the $lbp$ values are divided into $m=2$ parts by the histogram. Also, two end cells of vector $a$ are evaluated when the histogram divides the $lbp$ values into $m=4$ parts. This process can be expanded to include $i$ end cells of vector $a$ when the histogram divides the  $lbp$ values into $m=2^i$ equal parts for $i=1, 2, ...,8$. Consequently, when the histogram is applied, the significance of the cells diminishes from the end to the beginning. Specifically, the first four cells are disregarded when the histogram divides the  $lbp$ values into $m=16$ or less parts. 

This suggests that using 4-bit binary numbers, which encompass $4$ cells, is sufficient when applying the histogram process for feature reduction in the LBP descriptor. Reference \cite{martolia2020modified} empirically confirms this observation. This can be considered as a limitation of the standard LBP approach in dimension reduction, as it does not fully utilize the potential of the LBP descriptor. Therefore, the mathematical analysis in the next sections aims to address this issue, proposing an efficient algorithm to identify optimal LBP features for dimension reduction. The proposed algorithm overcomes histogram drawbacks and enhances the LBP feature extraction, efficiently. 

\begin{figure}[h!]
\centerline{
\includegraphics[scale=.5]{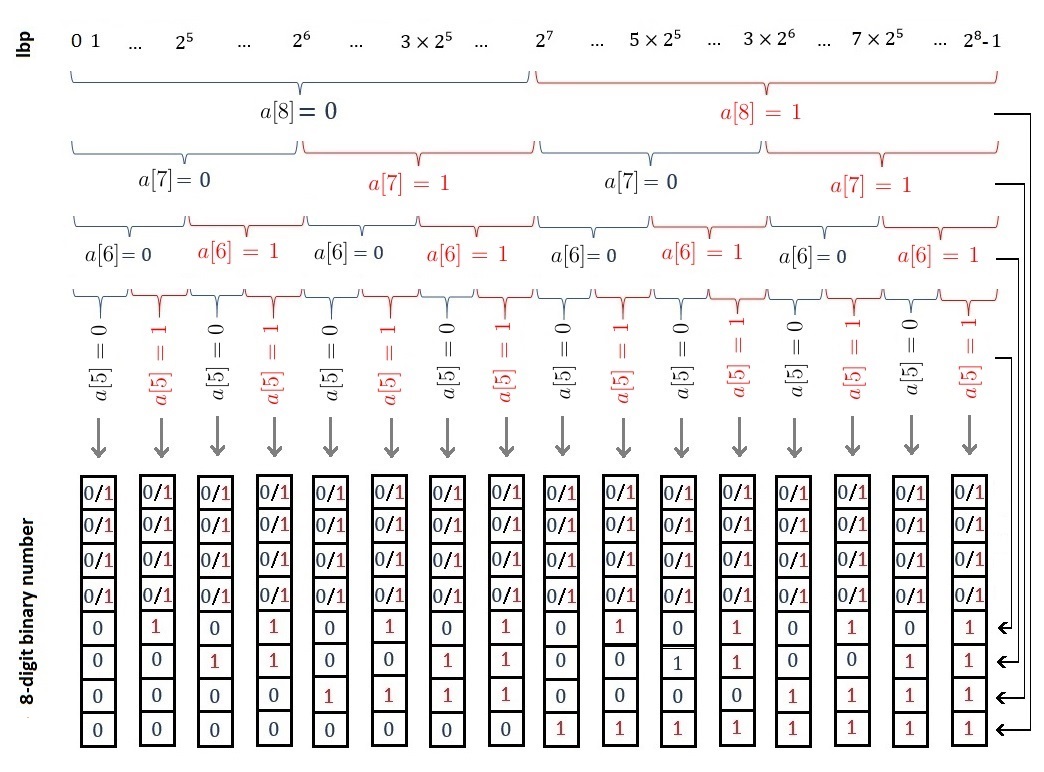}\hspace{0.5cm}
}
\caption{Function $lbp^{-1}$ transforms $lbp$ values to 8-bit binary numbers. One can see $i$ end cells of the numbers are only important when the histogram divides the $lbp$ values into $2^i$ equal sub-intervals for $i=1, 2, ...,8$.  }
\label{fig1}
\end{figure}

\section{Matrix representation of LBP feature extraction} \label{Sec3}
In this section, we analyze the LBP process and represent it in vector form. Vectorizing the LBP process involves breaking it down into three main parts: LBP  Matrix, $E$, Tiling matrix, $T$, and Histogram matrix, $H$. In fact, we construct three matrices $T, E$ and $H$ such that elements of matrix 
\be\label{F}
F = T \, E\, H ,
\ee
are the LBP features extracted from an image. These matrices are represented in the next subsections.

\subsection{LBP  Matrix}\label{Sec31}
In this subsection, our objective is to build a matrix based on the $lbp$ values defined in \eqref{lbp}. Assuming the image size is $n \times n$, we use a $3 \times 3$ local neighbourhood around each pixel of the image. As illustrated in Figure \ref{figVec}, we introduce an extension vector for each $lbp$ value. 
The extension vector is of length $256$  where all of its cells are zero except its $lbp$-th position, which it is evaluated equal to $1$. This vector is denoted by $\mathbf{e}$. Therefore, matrix $E$, which contains the extended vectors respect to the pixels in an image, is defined as follows:
\bes
E=
\begin{bmatrix}
        &   \mathbf{e}_1     &          \\
        &  \mathbf{e}_2     &           \\
        &   \vdots              &       \\   
        &   \mathbf{e}_n    &           \\
\end{bmatrix}_{n^2 \times 256} ,
\ees
where $\mathbf{e}_i$ is vector $\mathbf{e}$ at $i$-th pixel of the image.  Then $E$ is a matrix of size $n^2 \times 256$ when the size of the image is $n\times n$. Note that with an increased image size (e.g., through zero-padding), the size of matrix $E$ is exactly $n^2 \times 256$, not $(n-1)^2 \times 256$.
\begin{figure}[h!]
\centerline{
\includegraphics[scale=.25]{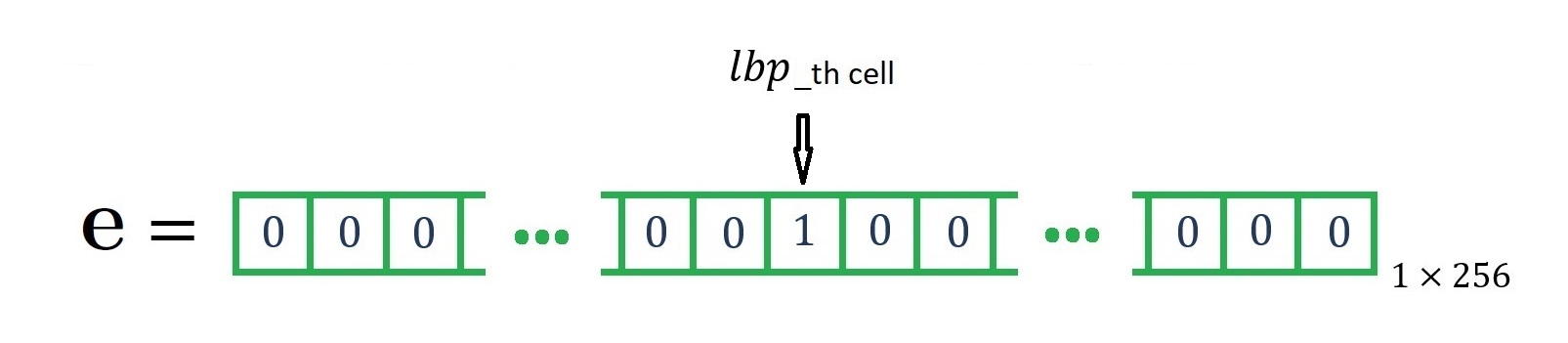}\hspace{0.5cm}
}
\caption{The extension vector for a LBP value. The vector takes 1 at $lbp$-th cell and 0, otherwise. }
\label{figVec}
\end{figure}

\subsection{Tiling matrix}\label{Sec32}
Tiling refers to the process of dividing an image into several local regions. Since finding global features may be less useful, we often divide an image into smaller regions and compute LBP features for each region separately. This approach enhances the efficiency of the LBP method. The main image is typically divided into $4, 9, 16$ and $25$ equal sub-regions in the LBP approach \cite{ojala2002multiresolution}. It is important to note that these 
sub-regions do not overlap. Suppose we want to divide the image into  $l$ sub-regions, with each sub-region denoted
as $\Omega_i$ where $i=1,2,..,l$. For example, as shown in Figure \ref{fig32}, the image is divided into 16 sub-regions, which
 can be numbered from left to right in the first row, and then from left to right in the second row, and so on.  The tiling matrix,  $T$, can be represented as:   
\bes 
T=
\begin{bmatrix}
        &   \mathbf{t}_1     &          \\
        &  \mathbf{t}_2     &           \\
        &   \vdots              &       \\   
        &   \mathbf{t}_l    &           \\
\end{bmatrix}_{l \times n^2} ,
\ees
where $j$-th element of vector  $\mathbf{t}_i$ is evaluated as
\begin{equation}\label{s}
\mathbf{t}_i(j)=
\left \{
\begin{array}{ll}
1, ~~~~~~~~~~~ \hbox{if} ~~ x_{j}\in\Omega_i,\\
\\
0, ~~~~~~~~~~~ o.w.
\end{array}\right.
\end{equation}
for $j=1,2,...,n^2$ when $x_j$ is the location of $j$-th pixel in the image.

\subsection{Histogram matrix}\label{Sec33}
Now, let's introduce the histogram matrix, $H$. Assuming we want to reduce the extracted features to $m$ for each sub-regions by the histogram approach, we divide the interval $[0, 255]$ into $m$ equally sized subintervals $V_i$ for $i=1, 2, ..., m$. The matrix $H$ is then defined as:
\bes
H=
\begin{bmatrix}
           \mathbf{h}_1^T, \mathbf{h}_2^T, \hdots \mathbf{h}_m ^T
\end{bmatrix}_{256 \times m} ,
\ees
where $T$ denotes the transpose operator and $j$-th cell of vector $\mathbf{h}_i$ equals $1$ when $j$ is in sub-interval $V_i$ when interval $[0,255]$ is divided to $m$ sub-intervals $V_1, V_2, ..., V_m$, and it is 0 otherwise, i.e. 
\begin{equation}\label{s}
\mathbf{h}_i(j)=
\left \{
\begin{array}{ll}
1, ~~~~~~~~~~~ \hbox{if} ~~ j\in V_i,\\
\\
0, ~~~~~~~~~~~ o.w.
\end{array}\right.
\end{equation}
\\

Therefore, we introduced three matrices $T, E$ and $H$. In the standard LBP context, matrix $T$ sums the number of repeated $lbp$ values for the pixel blocks located in specific regions, while matrix $H$ summarizes patterns by grouping them and performing equalization within each group, akin to a histogram formula. Thus, matrices $T$ and $H$ are fixed based on the desired number of features we wish to extract, whereas matrix $E$ depends on the image.
 When extracting standard LBP features from an image, assuming matrices  $T$ and $H$ are known, it suffices to compute matrix $E$. Then, by computing  $F = T \, E\, H$, the features are extracted from the image.  
 For example, if we divide an image of size  $128\times 128$ into $l=16$ sub-regions and aim to extract $m=8$ LBP features from each sub-region by the histogram, the matrices  have the following dimensions: 
 $T$ is  $16\times 128^2$, 
$E$ is $128^2\times 256$,
$H$ is $256\times 8$. 
Consequently, $F = T \, E\, H$ results in a matrix of size $16\times8$, where each row represents the features extracted from the corresponding sub-region. Then, totally $128=16\times 8$ features are extracted from the image as is shown in Figure \ref{fig32}.

\section{Motivation and novelty of the paper}\label{Sec4}
This study is motivated by the need to overcome the limitations of the standard LBP and improve the efficiency in identifying optimal features for dimension reduction and classification applications. 
By delving into the foundational principles and adopting the matrix representation of the LBP, as discussed in the previous section, we aim to optimize the LBP feature extraction. 
Earlier, we highlighted the inherent ambiguities in using histograms for LBP feature extraction. 
It is important to note that  the standard LBP matrix process employs two fixed transformation matrices,  $T$ and $H$, for feature extraction, which are not tailored to specific datasets or problems. 
These matrices are designed to extract a predetermined set of features irrespective of the dataset characteristics. 
Our primary objective is to develop two specialized transformation matrices that will significantly enhance the extraction of relevant features tailored to our specific problem domain. 
Moreover, we aim to enhance the effectiveness of feature extraction and mitigate the ambiguities associated with traditional histogram methods. 
This innovative approach ensures that our method remains robust and effective across diverse datasets, optimizes consistency and reliability in feature representation. 
To achieve these customized matrices, we propose to extend the matrix-based LBP process. 
While the matrix relation $F = T \, E\, H$ typically applies to individual images, we acknowledge the feasibility of deriving distinct transformation matrices for each image.
However, to maintain consistency with the standard LBP process, our aim is to keep these transformation matrices constant across all images in the dataset. To achieve this, we extend the relation to a more generalized form:
\begin{align}\label{AA}
\, \bar{F} \,  = L \, \bar{E} \, H , 
\end{align}
where $\bar{E}$ is a robust representation for the LBP matrices of the images in the dataset. Statistically, the mean of the LBP matrices obtained from the images in the dataset, i.e. ,  
\bes
\bar{E} = \frac{1}{N} \sum_{i=1:N} E_i ~,
\ees  
is a suitable representation. Here, $E_i$ is the LBP matrix of the $i$-th image, and $N$ is the total number of training images.
Note that, while $T$ and $H$ are fixed matrices for the images, we have $\bar{F}   = \frac{1}{N} \sum_{i=1}^{N} F_i$, where $F_i = T \,E_i \, H$. Then, in a mathematical sense, $\bar{F} $ is the mean of ${F}_1, F_2, ... , F_N $.
As mentioned before, in this paper, our primary focus is on optimizing the transformation matrices $T$ and $H$ not only for dimension reduction but also classification. We will explain the approach for two classes for simplicity, and subsequently, it can be extend it to multi-class classification using the One-vs-One (OvO) or One-vs-Rest (OvR) strategy\cite{bishop2006pattern}.

Let $\bar{E}_1$ and $\bar{E}_2$ denote the mean matrices respect to Class 1 and Class 2, respectively. Since the most effective features  are those that  distinguish between the classes, we can use the difference between these matrices,  i.e. ,
$
\bar{E}_{1,2}  = \bar{E}_2-\bar{E}_1,
$
for the classification. Accordingly, based on equation \eqref{AA} for feature extraction of $\bar{E}_{1,2} $, we have:
\bes
\, \bar{F}_{1,2} \,  = T \, \, \bar{E}_{1,2}  \,  \, H,
\ees 
where  $\bar{F}_{1,2}$ represents the LBP features corresponding to the difference between the classes. Our goal is to construct transformation matrices  $T$ and $H$  that transform  $\bar{E}_{1,2} $ into the final matrix $\bar{F}_{1,2}$, while preserving the most critical information distinguishing these two classes. In mathematical terms, the optimal $T$ and $H$ are those that map the matrix $\bar{E}_{1,2} $ to the largest possible matrix in terms of the rank and norm. Achieving this objective effectively involves employing the singular value decomposition algorithm, as elaborated in the following section. 

Hereafter, we focus on the classification problem because identifying optimal features, which lead to dimension reduction, can be considered a special case of classification by assuming $\bar{E}_2 = \bar{E}$ and $\bar{E}_1 = \bold{0}$ (where $\bold{0}$ is the zero matrix).

\section{Optimal transformation matrices }\label{Sec5}
In this section, we will utilize the SVD for optimizing the LBP transformation matrices, $T$ and $H$, introduced earlier for the classification. SVD is a powerful mathematical tool that decomposes a matrix into its constituent singular vectors and values. 
Since the SVD provides a decomposition that highlights the most significant components (singular vectors) of $\bar{E}_{1,2} $, it is particularly effective for this purpose. Thus, from SVD, let matrix $\bar{E}_{1,2}$  can be decomposed to three matrices, $U, \Sigma$ and $V$ as
\be\label{svd}
\bar{E}_{1,2}=U \Sigma V^T ,
\ee
where $\Sigma$ is a diagonal matrix containing the singular values in descending order. These singular values represent the magnitude of each corresponding component's contribution to the overall structure of the data. Matrices $U$ and $V$ are left and right singular vectors of $\bar{E}_{1,2}$ satisfy $V V^T=I$ and $U U^T=I$, where $I$ denotes the identity matrix \cite{strang2012linear,trefethen2022numerical}. Then we can find
\bes 
U=[\mbf{u}_1, \mbf{u}_2, ..., \mbf{u}_{n^2}]_{n^2\times n^2}~,~
\Sigma=
\begin{bmatrix}
       \sigma_1 & 0  &     0         &  \dots     & 0      \\
       0 			&  \sigma_2 &     0         &\dots       & 0     \\
       0           &     0         &  \sigma_3 &              &      \\
        \vdots   & \vdots     &                 & \ddots    &              \\   
       0           &    0          &             &            &       \\
\end{bmatrix}_{n^2 \times 256} ~,~
V=[\mbf{v}_1, \mbf{v}_2, ..., \mbf{v}_{256}]_{256\times 256}~,~
\ees
where $\mbf{u}_i$ and $\mbf{v}_i$ represent $i$-th column of matrices $U$ and $V$, respectively.  
Thanks to the SVD algorithm, the following theorem is valid.

\begin{Theorem}\label{th1}
For two normal column vectors $\mbf{u}$ and $\mbf{v}$ of length $n^2$ and $256$, respectively we have 
\bes
| \mbf{u}^T \, \bar{E}_{1,2} \, \mbf{v} | \leq | \mbf{u}_1^T \, \bar{E}_{1,2} \, \mbf{v}_1 | = \sigma_1,
\ees
\end{Theorem}
\begin{proof}
We know that columns of matrix $U$ are independent vectors that span the linear space $\mathbb{R}^{n^2}$ \cite{strang2012linear,trefethen2022numerical}. Therefore, vector $\mbf{u}$ can be expanded by the columns of $U$ as :
\be\label{u}
\mbf{u}=\sum_{i=1:n^2} \alpha_{i} \mbf{u}_i= U \boldsymbol{\alpha},
\ee
where $\mbf{u}_i$ denotes the $i$-th column of $U$ and $\alpha_{i}$ are the expansion coefficients, with $\sum_{i=1:n^2} \alpha_{i}^2=1$.
Let $\boldsymbol{\alpha}$ be a column vector evaluated as  $\boldsymbol{\alpha}=[\alpha_1, \alpha_2, ..., \alpha_n]^T$. Similarly, vector $\mbf{v}$ can be expanded by the columns of $V$ as 
\be\label{v}
\mbf{v}=\sum_{i=1:256} \beta_{i} \mbf{v}_i= V \boldsymbol{\beta},
\ee
where $\mbf{v}_i$ denotes the $i$-th column of $V$ and $\beta_{i}$ are the expansion coefficients, with
$\sum_{i=1:256} \beta_{i}^2=1$. Let $\boldsymbol{\beta}$ be a column vector evaluated as  $\boldsymbol{\beta}=[\beta_1, \beta_2, ..., \beta_{256}]^T$.  Now from SVD \eqref{svd} and equations \eqref{u} and \eqref{v} we have
\bes
\mbf{u}^T \, \bar{E}_{1,2} \, \mbf{v} = \boldsymbol{\alpha}^T U^T (U \Sigma V^T) V \boldsymbol{\beta} =\boldsymbol{\alpha}^T \Sigma \boldsymbol{\beta} = \sum_{i=1:256} \sigma_i \alpha_i \beta_i , 
\ees 
which it is maximized for $\alpha_1=\beta_1=1$. Consequently, the theorem is valid. 
\end{proof}
Then the following corollary is obtained
\begin{cor}\label{o1}
Thanks to the SVD and Theorem \ref{th1}, maximum value of $| \, \bar{F}_{1,2} \,  |$ is obtained for 
\be\label{oS}
T=\mbf{u}_1^T~,~~ H=\mbf{v}_1~,
\ee
when $l=m=1$. 
\end{cor}

Corollary \ref{o1} denotes the optimal values of matrices $T$ and $H$ that lead to the maximum value of the distance between $\bar{E}_1$ and $\bar{E}_2$ when a linear combination of the LBP features is extracted from the LBP matrix. This fact can be extended to higher dimension and find more vectors distinguish two classes more efficient. The next theorem states this fact.
\begin{Theorem}
Let $T$ and $H$ are two matrices with orthonormal rows and columns, respectively. Then 
\bes
\| T \, \bar{E}_{1,2} \, H \| \leq \| \tilde{U}^T \, \bar{E}_{1,2} \, \tilde{V} \| = \sum_{i=1:\min\{l,m\}} \sigma_i,
\ees
for $l,m \leq 256$, where $\tilde{U}$ is a matrix obtained from the first $l$ columns of matrix $U$, and $\tilde{V}$ is a matrix obtained from the first $m$ columns of matrix $V$, evaluated as 
\be\label{tU}
\tilde{U}=[\mbf{u}_1, \mbf{u}_2, ..., \mbf{u}_l]~,~ \tilde{V}=[\mbf{v}_1, \mbf{v}_2, ..., \mbf{v}_m],
\ee
\end{Theorem}
\begin{proof}
By induction on $l$ and $m$ and using the SVD algorithm alongside Theorem \ref{th1}, the theorem can be demonstrated.
\end{proof}

Note that, since matrices $\tilde{U}^T$  and $\tilde{V}$ are orthonormal, using them as transformation matrices $T$ and $H$ results in a diagonal matrix $\, \bar{F}_{1,2} \, $. Therefore, diagonal elements of matrix $\, \bar{F}_{1,2} \,$ can be applied to extract important features from $\, \bar{E}_{1,2} \,$ reflect the differences between the classes.
Consequently, the following corollary is valid. 
\begin{cor}
The first $i$  diagonal elements of matrix $ \, \bar{F}_{1,2} \, =\tilde{U}^T \, \bar{E}_{1,2} \, \tilde{V}$ can be supposed as more efficient features highlighting the differences between matrices $\bar{E}_1$ and $\bar{E}_2$ when $l=m \leq 256$. 
\end{cor}

\section{Optimal LBP values}\label{Sec6}
In this section, our objective is to identify the most effective LBP values for classification tasks. We focus on extracting fundamental LBP values that robustly differentiate between two classes, leveraging insights from SVD. By refining vectors derived from SVD, we aim to pinpoint key LBP values that enhance classification accuracy and robustness across diverse datasets. 
Algorithm \ref{alg2} is introduced to achieve this goal by selecting and refining optimal LBP values based on their contributions to the vectors obtained from SVD. 
\begin{algorithm}
\caption{: Optimal Transform Matrices.}\label{alg2}
\begin{algorithmic}[1]
\State \textbf{Input:} Matrix $\bar{E}_{1,2}$, integers $l=m$, $n$.
\State Initialize matrices $H = \text{zeros}(256, m)$ and $T = \text{zeros}(l, n^2)$.
\For{$i = 1$ to $m$}
    \State Apply SVD to $\bar{E}_{1,2}$ and find $\mathbf{v}_1$ from Corollary \ref{o1}.
    \State Initialize $\mathbf{h}_i$ as $\mathbf{h}_i = \text{zeros}(256, 1)$.
    \State Set $[\text{max}, \text{ind}] = \text{max}(\text{abs}(\mathbf{v}_1))$.
    \State Set $\mathbf{h}_i(\text{ind}) = 1$.
    \State Set $\mathbf{t}_i = \bar{E}_{1,2} \mathbf{h}_i$.
    \State Set $\mathbf{t}_i = \mathbf{t}_i / \|\mathbf{t}_i \|$.
    \State Set $H(:, i) = \mathbf{h}_i$ and $T(i, :) = \mathbf{t}_i^T$.
    \State Update $\bar{E}_{1,2}$ as $\bar{E}_{1,2} \leftarrow \bar{E}_{1,2} - (\bar{E}_{1,2} \mathbf{h}_i) \mathbf{h}_i^T$.
\EndFor
\State \textbf{Output:} Matrices $H$ and $T$.
\end{algorithmic}
\end{algorithm}

In this algorithm, LBP values that contribute the most to the SVD vectors are selected as optimal. The resulting refined LBP vectors are denoted as $\mbf{h}_i$ for $i=1, 2, ..., m$. These vectors yield left transform vectors $\mbf{t}_i$, represent the projection of matrix $\bar{E}_{1,2} $ onto $\mbf{h}_i$. The new algorithm sets the stage for improving classification outcomes by focusing on the most discriminative LBP values identified through SVD-based refinement.
\begin{figure}[h!]
\centerline{
\includegraphics[scale=.22]{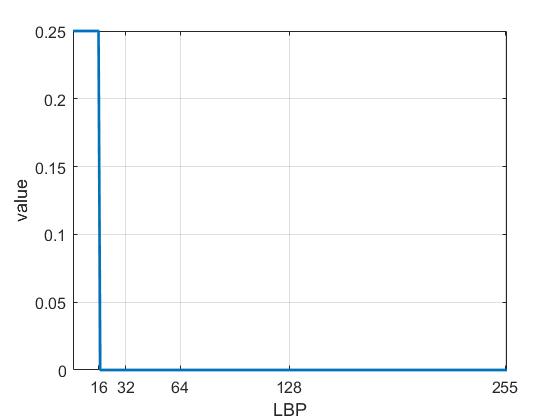}\hspace{0.15cm}
\includegraphics[scale=.22]{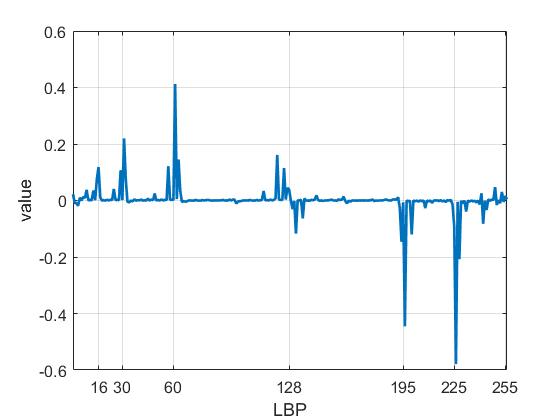}\hspace{0.15cm}
\includegraphics[scale=.22]{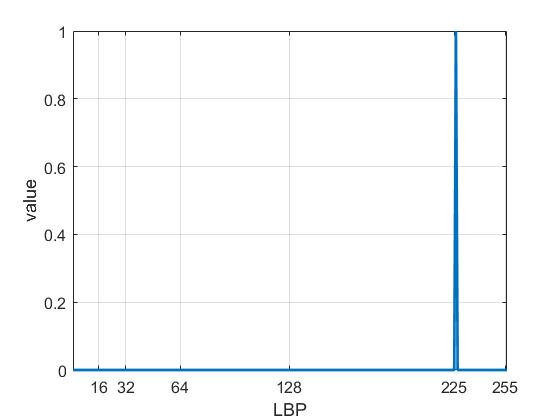}\hspace{0.15cm}
}
\centerline{
\includegraphics[scale=.22]{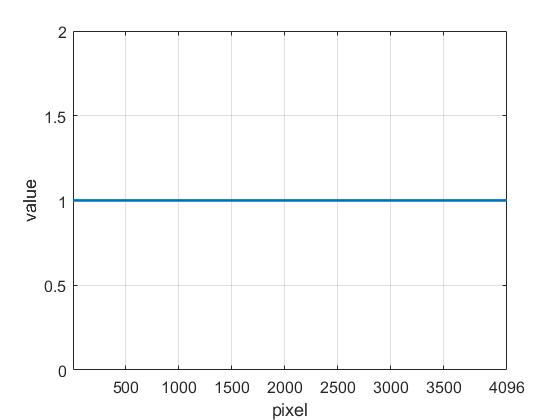}\hspace{0.15cm}
\includegraphics[scale=.22]{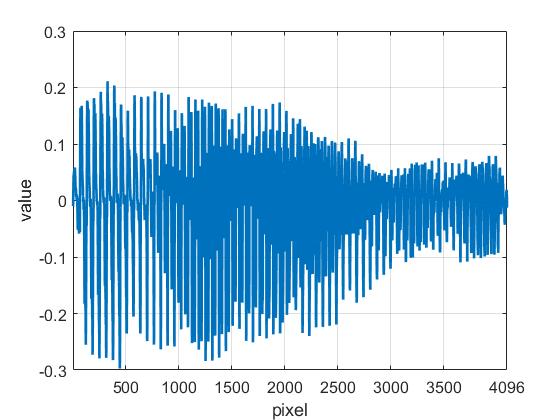}\hspace{0.15cm}
\includegraphics[scale=.22]{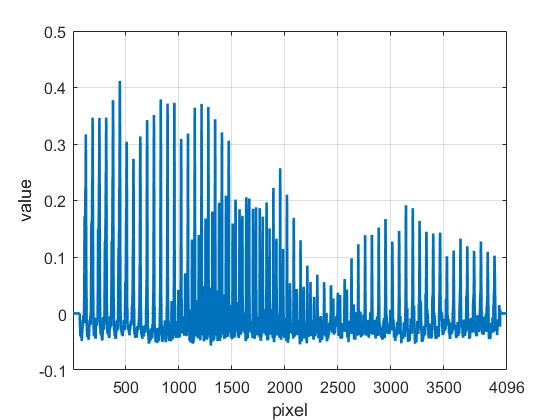}\hspace{0.15cm}
}
\caption{The first vector of the transform matrices, $T$ and $H$, for the standard LBP, SVD, and the proposed LBP are presented from left to right. The first and second rows are respect to the values of $H$ and $T$ matrices, respectively.}
\label{fig_sec6}
\end{figure}
In Figure \ref{fig_sec6}, the first vector of the left and right transform matrices is shown for the standard LBP, the SVD algorithm and the new proposed algorithm for the face detection classification problem. In the figure, the first row corresponds to the right transform and the second row corresponds to the left transform vectors.  From left to right, column 1 is for the standard LBP, column 2 is for the SVD algorithm, and the last column is for the new algorithm. 
The face images are sourced from databases CFD [55], CFD-MR [56], and CFD-INDIA [57], where the facial regions were manually cropped and resized to images of size $64 \times 64$. This combined database is referred to as CFD-T in this paper. Additionally, the clutter images are extracted from a database that contains no human faces. Some representative face and clutter images from these databases are shown in Figure \ref{fig6}. The studied images are grayscale, and the light intensity of pixels in the images is normalized by dividing their values by $255$.

As seen in  Figure \ref{fig_sec6}, in the standard LBP, the transformation matrices are fixed and independent of the classification. 
However, they vary with classification for both the SVD and the proposed algorithm. In SVD, LBP values  $224, 195$ and $60$ contribute most significantly to the right transform matrix, $H$. According to Algorithm \ref{alg2}, $lbp=224$ has the greatest impact on the first eigenvector of the SVD algorithm, making it the first optimal LBP value selected.
Then, the first vector of the right transform matrix in the proposed algorithm, has only one non-zero pixel, equals $1$ at the $224$-th position. Additionally, the first vector of the left transformation matrix, $T$, is predominantly positive for the new algorithm, whereas it exhibits both positive and negative values for the SVD. This highlights that the optimal LBP value focuses on more important facial features and reduces clutter.

\begin{figure}[htbp]
\centerline{
\includegraphics[scale=.33]{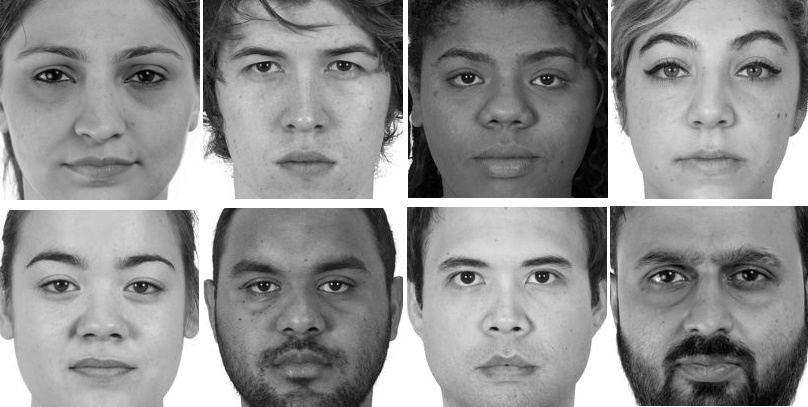}\hspace{0.5cm}
\includegraphics[scale=.33]{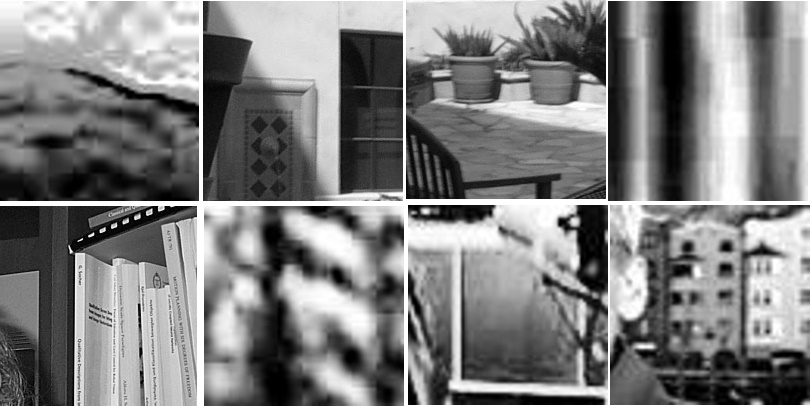}
}
\caption{Some human frontal face (left) and clutter (right) images in the face and clutter databases, respectively.}
\label{fig6}
\end{figure}

\begin{figure}[h!]
\centerline{
\includegraphics[scale=.36]{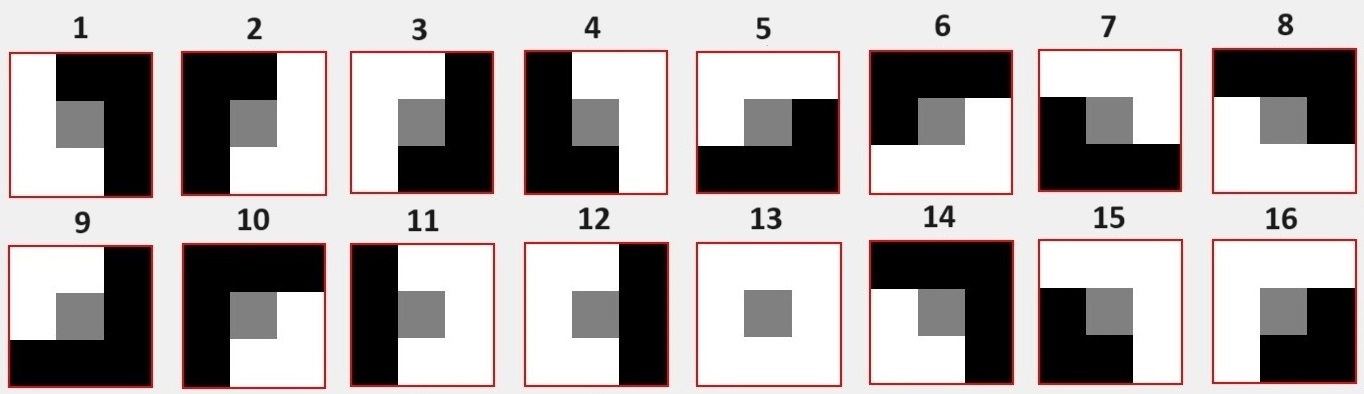}\hspace{0.5cm}
}
\caption{The 16 first LBP values in form of 8-bit binary number for face detection. The value of black and white cells are 0 and 1, respectively.}
\label{fig7}
\end{figure}

The $16$  most important LBP values are shown in Figure \ref{fig7}  in 8-bit binary numbers form. These were obtained from algorithm \ref{alg2} for  dataset CFD-T.  The white and black cells in Figure \ref{fig7} refer to 1 and 0 values. One can observe the selected LBP features are mostly paired, such as the  first and forth, the second and third, the fifth and sixth, and the seventh and eighth. Additionally, the black and white pixels are not dispersed much,  indicating uniform LBP values \cite{ojala2002multiresolution, khaleefah2020review}. 

To compare the effectiveness of the new optimal LBP features with the standard LBP features and those obtained by the SVD in dimension reduction, the residual norm of matrix $\bar{E}_{1,2} $, denoted as
\bes
R=\| \bar{E}_{1,2}  -  (\bar{E}_{1,2}   \, H) \, H^T  \| ,
\ees
is calculated for $m=1,2,..., 256$ and presented in Figure \ref{figNorm}. The residual diminishes slowly for the standard LBP, whereas it diminishes more rapidly for the other two algorithms, as observed in Figure \ref{figNorm}.
\begin{figure}[h!]
\centerline{
\includegraphics[scale=.4]{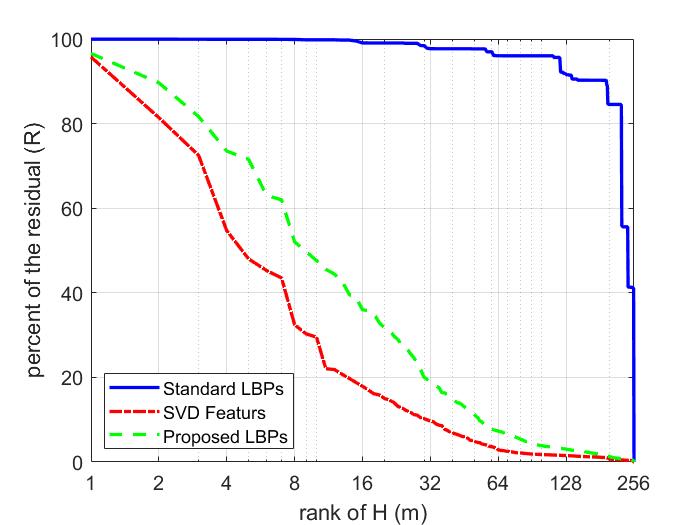}\hspace{0.5cm}
}
\caption{The percent of the residual error of $\bar{E}_{1,2} $ in field of the feature numbers, $m$, extracted by the standard LBP, SVD and the proposed algorithm.}
\label{figNorm}
\end{figure}

The fast reduction of the residual error leads to more efficient feature extraction and consequently more accurate classification results. In Figure \ref{fig9}, histograms of feature values extracted from face and clutter images of dataset CFD-T are shown for one and two dimensions. It is evident that the distributions are more effectively separated in SVD and the new algorithm. Note that the standard LBP is unable to distinguish between face and clutter images in these situations and requires more features to do so; typically, we use  $64$ or more features for standard LBP. 
\begin{figure}[h!]
\centerline{
\includegraphics[scale=.22]{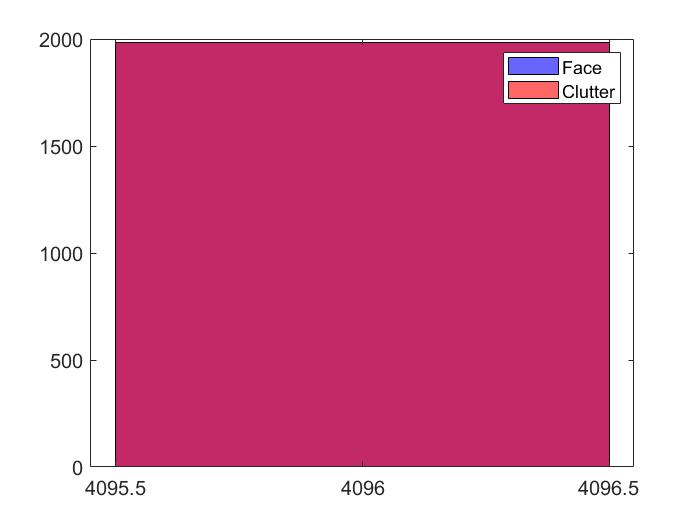}\hspace{0.15cm}
\includegraphics[scale=.22]{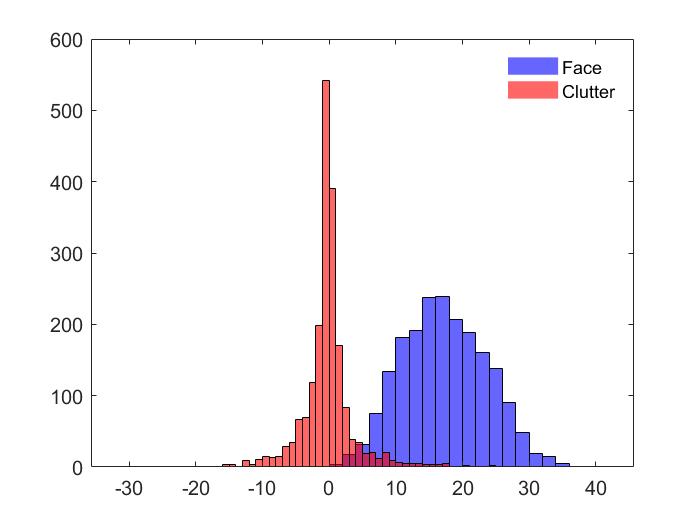}\hspace{0.15cm}
\includegraphics[scale=.22]{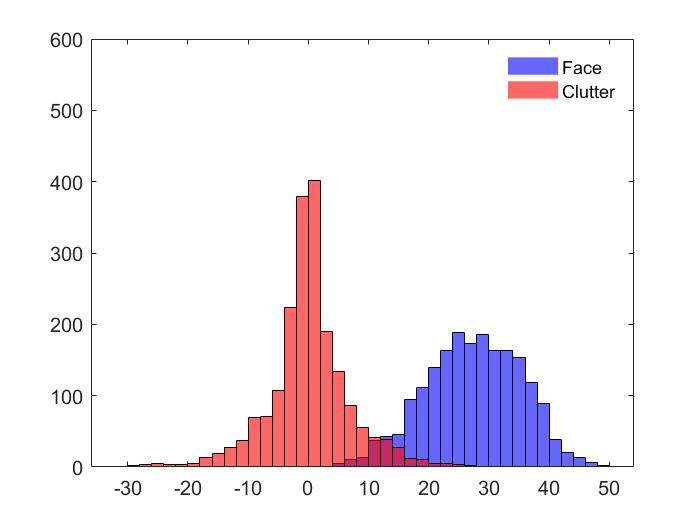}\hspace{0.15cm}
}
\centerline{
\includegraphics[scale=.27]{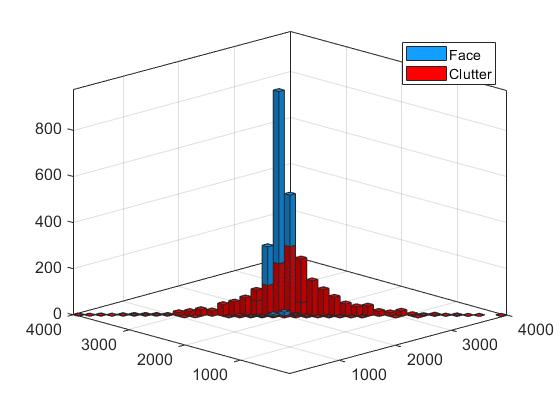}\hspace{0.15cm}
\includegraphics[scale=.27]{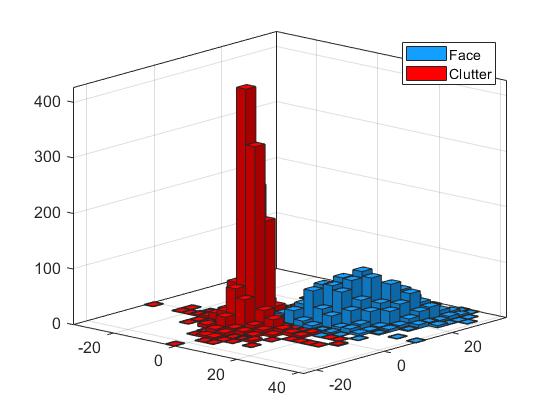}\hspace{0.15cm}
\includegraphics[scale=.27]{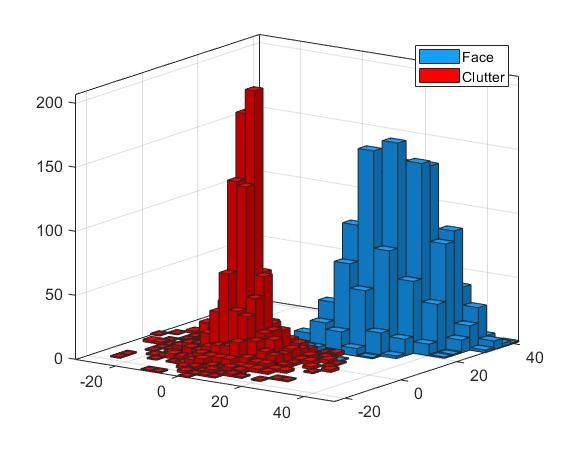}\hspace{0.15cm}
}
\caption{The histogram of face and clutter images for one and two dimensions. The images are respect to the standard LBP, SVD and proposed LBP features from left to right, respectively. SVD and the proposed algorithms are able to extract features that effectively distinguish between face and clutter images, whereas the standard LBP approach fails to do so.}
\label{fig9}
\end{figure}

\section{Numerical experiments}\label{Sec7}
This section is devoted to the numerical experiments for face detection and facial expression recognition. First, the accuracy rate of the new optimal LBP features, obtained from Algorithm \ref{alg2}, is reported in Subsection \ref{Sub7.1}, using the CFD-T dataset for face detection. In Subsection \ref{Sub7.2}, the optimal LBP features are extracted for facial expression recognition using the CK-Data dataset\cite{lucey2010extended}. In all case studies, $70$ percent of images in the dataset are used as training images and $30$ percent as testing images. The LBP features are extracted from the training images and applied to the testing images. Consequently, the accuracy is reported only for the testing images when the linear SVM is used for classification.

\subsection{face detection}\label{Sub7.1}
The accuracy of the classification for standard LBP, standard SVD, and the proposed LBP features is reported in Figure \ref{fig7.1} in terms of the number of features. From Figure \ref{fig7.1}, the accuracy of the classification is more than $99$ percent for face and clutter images when using standard SVD and the proposed LBP features, provided the number of features is more than or equal to $16$. However, the accuracy of the standard LBP is significantly lower for small numbers of features. This demonstrates the efficiency of the proposed LBP features, especially for small numbers of features.

\begin{figure}[h!]
\centerline{
\includegraphics[scale=.4]{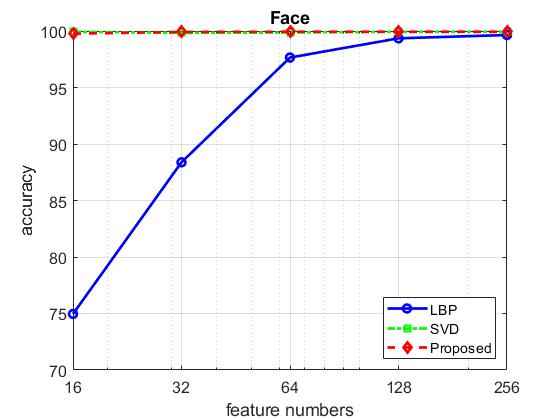}\hspace{0.5cm}
\includegraphics[scale=.4]{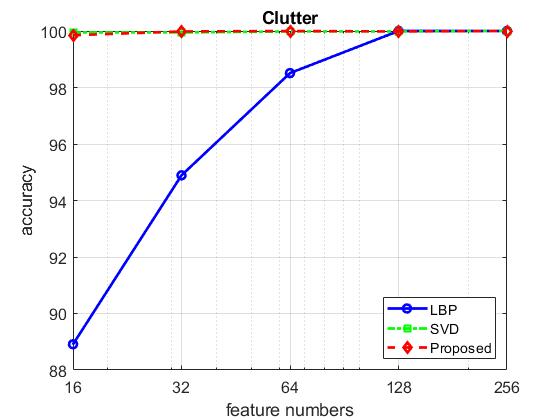}\hspace{0.5cm}
}
\caption{The accuracy of face and clutter classification.}
\label{fig7.1}
\end{figure}
We also apply the LBP feature extraction to a dataset containing 1,500 randomly selected faces from the UTKFace dataset \cite{zhifei2017cvpr} and 1,500 clutter images. The accuracy of the SVM classification exceeds $98\%$ for face and clutter images when using only $16$ new LBP features extracted by Algorithm \ref{alg2}. In contrast, this level of accuracy is achieved by the standard LBP only when $256$ features are extracted. This further confirms the efficiency of the proposed LBP features for classification.

\subsection{facial expression recognition}\label{Sub7.2}
In this subsection, we present the outcomes of our experiments implemented on the CK dataset. The dataset comprises 981 grayscale images of facial expressions, each measuring $48\times48$ pixels. These images are annotated into seven distinct emotional classes: anger (135 samples), contempt (54 samples), disgust (177 samples), fear (75 samples), happiness (207 samples), sadness (84 samples), and surprise (249 samples). In standard LBP, the images are partitioned into 16 regions of equal size, and a feature vector is generated by concatenating the LBP histograms of these regions, resulting in 256 features (see Figure \ref{fig32}). Additionally, 16 optimal LBP values are extracted for each facial expression using Algorithm \ref{alg2}. The optimal LBP values (in form of 8-bit binary numbers) are shown in Figure \ref{fig7.2}. Note that most of the optimal LBP values are duplicated across several expressions. Also, many of them, especially the first eight, are presented in Figure \ref{fig7} as the optimal LBP values for the face detection.

%

\begin{figure}[h!]
\centerline{
\includegraphics[scale=.38]{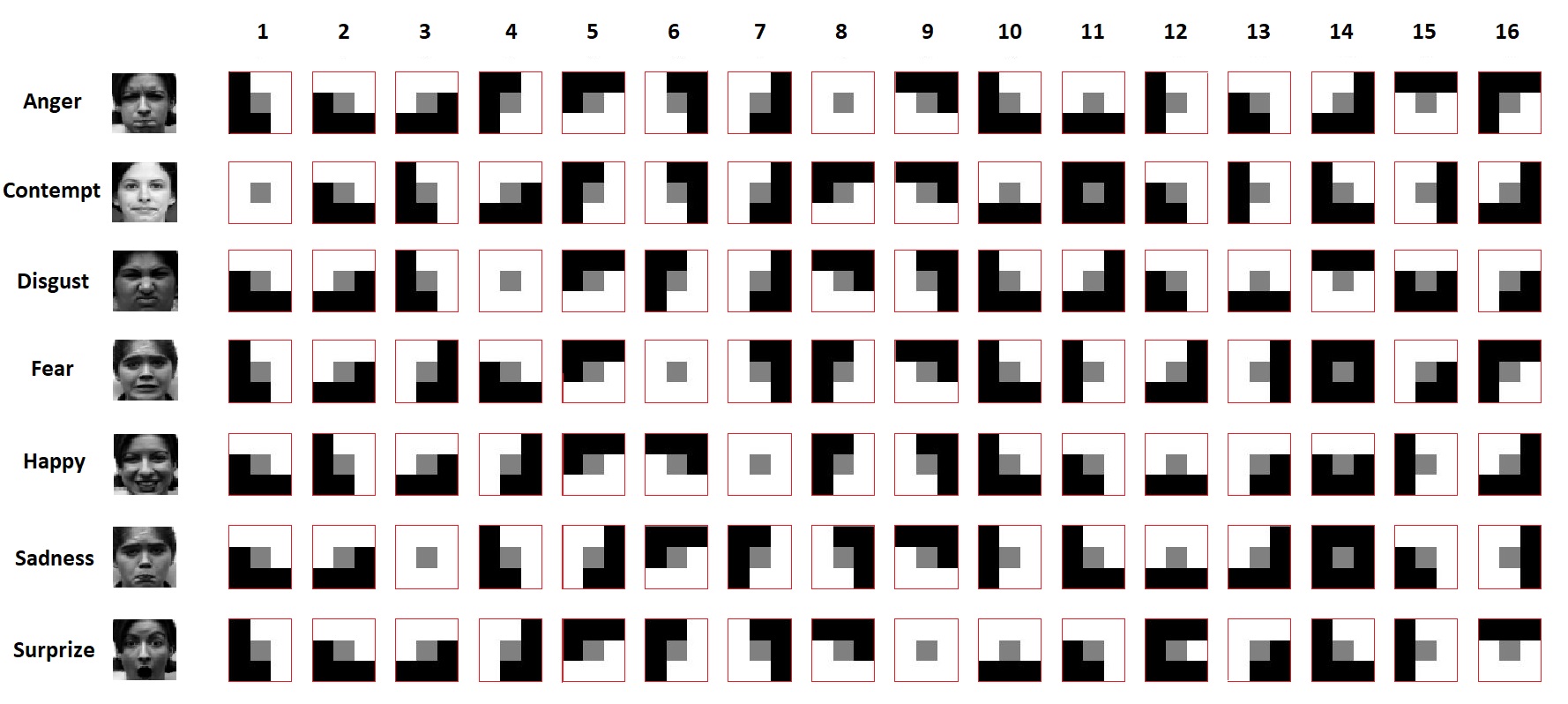}\hspace{0.5cm}
}
\caption{Optimal LBP values for facial expression recognition. The value of black and white cells are 0 and 1, respectively.}
\label{fig7.2}
\end{figure}
Then, the patterns extracted from the images were utilized for feature extraction by Algorithm \ref{alg2} and classification by SVM, with 70 percent of the images allocated for training and 30 percent for testing. Linear SVM is employed in the classification process. The confusion matrix of the classification is presented in Table \ref{tab:table2} for the standard LBP features and Table \ref{tab:table1} for the new LBP features.
\begin{table}[h!]
  \begin{center}
    \caption{Confusion matrix for the proposed LBPs.}
    \label{tab:table1}
    \begin{tabular}{ccccccccc}
    \hline
	\hline
	  class&  1 &  2 &  3 &  4 &  5 &  6 &  7 & ~~~~~~~~ \\
	\hline
      1    & 36 &  0 &  1 &  0 &  0 &  0 &  0 &  \\
      2    &  0 & 17 &  0 &  0 &  0 &  0 &  0 &  \\
      3    &  0 &  0 & 53 &  0 &  0 &  0 &  0 &  \\
      4    &  0 &  0 &  0 & 21 &  0 &  0 &  0 &  \\
      5    &  0 &  0 &  0 &  0 & 63 &  0 &  0 &  \\
      6    &  0 &  0 &  0 &  0 &  0 & 25 &  0 &  \\
      7    &  0 &  0 &  0 &  0 &  0 &  0 & 73 &  \\
      --   &  4 &  0 &  1 &  2 &  0 &  1 &  2 &   \\
     \hline
	 \hline
     \end{tabular}
  \end{center}
\end{table}
\begin{table}[h!]
  \begin{center}
    \caption{Confusion matrix for the standard LBP.}
    \label{tab:table2}
    \begin{tabular}{ccccccccc}
    \hline
	\hline
	 class &  1 &  2 &  3 &  4 &  5 &  6 &  7 & ~~~~~~~~\\
	\hline
    1      & 29 &  0 &  2 &  0 &  0 &  1 &  0 &  \\
    2      &  0 & 15 &  0 &  0 &  0 &  0 &  0 &  \\
    3      &  0 &  0 & 48 &  0 &  0 &  0 &  0 &  \\
    4      &  0 &  0 &  0 & 21 &  0 &  1 &  0 &  \\
    5      &  0 &  0 &  0 &  0 & 62 &  0 &  1 &  \\
    6      &  0 &  0 &  0 &  0 &  0 & 25 &  0 &  \\
    7      &  0 &  0 &  0 &  0 &  0 &  0 & 73 &  \\
    --     &  9 &  2 &  6 &  1 &  0 &  1 &  2 &   \\
     \hline
	 \hline
     \end{tabular}
  \end{center}
\end{table}

From the confusion matrices, the accuracy of the linear SVM classification for the new LBP features exceeds that of the standard LBP features. This improvement occurred despite the number of the new LBP features is less than half the number of features of the standard LBP.  We rune the program 20 times with random training and testing images, finding an average accuracy of approximately $90\%$ and $87\%$ for the proposed and standard LBP features, respectively. This demonstrates the enhanced accuracy of the proposed LBP features for facial expression recognition.

\section{Conclusion}\label{Sec8}
In this paper, we investigated the local binary pattern (LBP) process and proposed a novel algorithm based on singular value decomposition (SVD) to identify optimal LBP values for classification tasks. We extracted these optimal LBP values for applications in face detection and facial expression recognition, highlighting their effectiveness in classification. The optimal LBP values identified are uniform LBPs, meaning that the binary patterns exhibit at most two transitions between 0 and 1 or 1 and 0. This new strategy can be readily applied to other classification problems as well. In our approach, LBP features were extracted from the mean value of images within the datasets. However, more effective optimal LBPs can be achieved by utilizing a more appropriate linear combination of images instead of the mean value. Additionally, the optimization concept can be extended to larger pixel blocks, such as 5×5 and 7×7, allowing for the extraction of optimal LBP features in these contexts. Furthermore, this methodology can be adapted to other texture descriptors, such as Haar-like features, enabling further optimization.
\bibliographystyle{elsarticle-num}
\bibliography{mybibfile}

\newpage
\section*{Declaration}
\subsection*{Availability of Data and Materials}
Some datasets referred in the paper are analyzed during the current study are available. All materials used in this study are available for review and can be accessed upon request.
\subsection*{Competing Interests}
The authors declare that they have no competing interests. There are no financial or personal relationships that could influence the work reported in this paper.
\subsection*{Funding}
No funding was received for this study. 
\subsection*{Authors' Contributions}
Author 1: Conceptualization and methodology. \\
Author 2: Data analysis and writing original draft.\\
Author 3: Data analysis and writing programs.\\
All authors read and approved the final manuscript.
\subsection*{Acknowledgements}
The authors wish to acknowledge the contributions of all individuals and organizations involved in the research process, without naming specific persons or entities. Their collective efforts have been invaluable to the completion of this study.

\end{document}